\def\year{2020}\relax
\documentclass[letterpaper]{article} 
\usepackage{aaai20}  
\usepackage{times}  
\usepackage{helvet} 
\usepackage{courier}  
\usepackage[hyphens]{url}  
\usepackage{graphicx} 
\usepackage{graphicx}  
\newcommand{\ignore}[1]{}
\newcommand\eqdef{\stackrel{\mathclap{\normalfont\mbox{\normalfont\tiny def}}}{=}}

\usepackage{epsfig}
\usepackage{amsmath,amsthm,amssymb}

\usepackage{mathtools}
\usepackage{bm}
\usepackage{xcolor}

\newtheorem{proposition}{Proposition}
\newtheorem{observation}{Observation}
\usepackage{multirow}
\usepackage[ruled]{algorithm2e}

\title{Regularized Wasserstein Means for Aligning Distributional Data}
\author{
\Large \textbf{Liang Mi, Wen Zhang, and Yalin Wang}\\
Arizona State University\\
\{liangmi, wzhan139, ylwang\}@asu.edu 
}

\begin{document}
\maketitle

\begin{abstract}
We propose to align distributional data from the perspective of Wasserstein means. We raise the problem of regularizing Wasserstein means and propose several terms tailored to tackle different problems. Our formulation is based on the variational transportation to distribute a sparse discrete measure into the target domain. The resulting sparse representation well captures the desired property of the domain while reducing the mapping cost. We demonstrate the scalability and robustness of our method with examples in domain adaptation, point set registration, and skeleton layout.
\end{abstract}

\section{Introduction}

Aligning distributional data is fundamental to many problems in machine learning. From the early work on histogram manipulation, e.g.~\cite{stark2000adaptive}, to the recent work on generative modeling, e.g.~\cite{beecks2011modeling}, researchers have proposed various alignment techniques that benefit numerous fields including domain adaptation, e.g.~\cite{sun2016deep}, and shape registration, e.g.~\cite{ma2016non}. A universal approach to aligning distributional data is through optimizing an objective function that measures the loss of the map between them. Regarding one distribution as the fixed target and the other the source, the alignment process in general follows an iterative manner where we alternatively update their correspondence and transform the source. When the source has much fewer samples or in a lower dimension, the process is essentially finding a sparse representation~\cite{bengio2013representation}.

\ignore{New problem are defined, new theories are discovered, and new algorithms are proposed.}

The optimal transportation (OT) loss, or the Wasserstein distance, has proved itself to be superiors in many aspects over several other distances~\cite{gibbs2002choosing,arjovsky2017wasserstein}, benefiting various learning algorithms. By regarding the Wasserstein distance as a metric, researchers have been able to compute a sparse \textit{mean}~\cite{ho2017multilevel} of a distribution, which is a special case of the \textit{Wasserstein barycenter} problem~\cite{agueh2011barycenters} when there is only one target distribution. While OT algorithms find the correspondence between the distributions, updating the mean can simply follow the rule that each source sample is mapped to the weighted mean of its corresponding target sample(s)~\cite{ye2017fast}.

In this paper, we raise the problem of regularizing the \textit{Wasserstein means}. In addition to finding a mean that yields the minimum transportation cost, in many cases we also want to insert certain properties so that it satisfies other criteria. A common technique is adding regularization terms to the objective function. While most of the existing work, e.g.~\cite{cuturi2013sinkhorn,courty2017optimal}, focus on regularizing the optimal transportation itself, we address the mean update rule and show the benefit from regularizing it. We introduce a new framework to compute OT-based sparse representation with regularization. We base our method on variational transportation~\cite{mi2018variational} which produces a map between the source and the target distributions in a many-to-one fashion. Different from directly mapping the source into the weighted average of its correspondence~\cite{ye2017fast,courty2017optimal,mi2018variational}, we propose to regularize the mapping to cope with specific problems -- domain adaptation, point set registration, and skeleton layout. The resulting mean, or centroid, can well represent the key property of the distribution while maintaining a small reconstruction error. Code is available at \url{https://github.com/icemiliang/pyvot}

\section{Related Work} \label{sec:rw}
\subsection{Optimal Transportation}\label{sec:rw_rot}
The optimal transportation (OT) problem was raised introduced by Monge~\cite{monge1781memoire} in the 18th century, which sought a mass-preserving transportation map between distributional data with the minimum cost. It resurfaced in 1940s when Kantorovich~\cite{kantorovich1942translocation} introduced a relaxed version where mass \textit{can} be split and provided the classic linear programming solution. A breakthrough for the mass-preserving, or non-mass splitting, OT happened in the early 1990s when Brenier~\cite{brenier1991polar} proved its existence under quadratic Euclidean cost. In more recent years, fast algorithms for computing, or approximating, OT have been proposed in both lines of research -- non-mass-preserving, e.g.~\cite{rabin2011wasserstein,cuturi2013sinkhorn,solomon2015convolutional} and mass-preserving, e.g.~\cite{merigot2011multiscale,levy2015numerical,kolouri2016sliced,chen2019gradual}.

We follow Monge's mass-preserving formulation. Specifically, we adopt~\cite{mi2018variational} with improvements to compute the OT because it gives us a clear path of each sample, not a spread-out map. Thus, we can directly regularize the support instead of the mapping.

\subsection{Wasserstein Barycenters and Means}
The Wasserstein distance is the minimum cost induced by OT. In most cases, the cost itself may not be as desired as the map, but it satisfies all metric axioms~\cite{villani2003topics} and thus often serves as the loss for matching distributions, e.g.~\cite{ling2007efficient,arjovsky2017wasserstein}. Moreover, given multiple distributions, one can find their weighted average with respect to the Wasserstein metric. This problem was studied in~\cite{mccann1997convexity,ambrosio2008gradient} for averaging two distributions and generalized to multiple distributions in~\cite{agueh2011barycenters}, which coins the \textit{Wasserstein barycenter} term.

A special case of the barycenter problem is when there is only one distribution and we want to find its sparse discrete barycenter. Because computationally it is equivalent to the \textit{k-means} problem, \cite{ho2017multilevel} defines it as the \textit{Wasserstein means} problem. Before that, Cuturi and Doucet had discussed it in~\cite{cuturi2014fast} along with the connection of their algorithm to Lloyd's algorithm in that case. \cite{mi2018variational} proposes an OT-based clustering method which is very close to the Wasserstein means problem. \cite{kolouri2018sliced} also made a contribution by discussing the \textit{sliced Wasserstein Means} problem.

Our work focuses on \textit{regularizing} the Wasserstein means. We obtain the mean by mapping the sparse points into the target domain according to the OT correspondence. We insert regularization into the mapping process so that the sparse points not only have a small OT loss but they also have certain properties induced by the regularization terms. 

Our work should not be confused with other work on regularizing OT. For example, \cite{cuturi2013sinkhorn} introduces entropy-regularized OT where the entropy term controls the sparsity of the map and it was later used in~\cite{cuturi2014fast} to compute Wasserstein barycenters. \cite{courty2017optimal} also leveraged class labels to regularize OT for domain adaptation. \cite{ferradans2014regularized} proposed Sobolev norm-based regularized OT and further regularized barycenter and yet the regularization is still added to the OT, not the barycenter. These works only regularize OT and then directly update the support simply to the average of its correspondence. In this paper, we regularize the update.


\section{Preliminaries}\label{sec:p}
We begin with some basics on optimal transportation (OT). Suppose $M$ is a compact metric space, $\mathcal{P}(M)$ is the space of all Borel probability measures on $M$ and $\mu, \nu \in \mathcal{P}(M)$ are two such measures. A measure in the product space, $\pi(\cdot,\cdot) \in \mathcal{P}(M \times M)$, serves as a mapping between any two measures on $M$, i.e. $\pi\colon M\rightarrow M$. We define the cost function of the mapping as the geodesic distance $c(\cdot,\cdot)\colon M \times M \rightarrow \mathbb{R}^{+}$.

\subsection{Optimal Transportation}\label{sec:p_ot}
For a mapping $\pi(\mu,\nu)$ to be legitimate, the push-forward measure of one measure has to be the other one, i.e. $\pi_\#\mu = \nu$. Thus, for any measurable subsets $B,B' \subset M$ we have $\pi(B \times M) = \mu(B)$ and $\pi(M \times B') = \nu(B')$. We denote the space of all legitimate product measures by $\Pi(\mu,\nu) = \{\pi \in \mathcal{P}(M \times M)\ |\ \pi(\cdot,M) = \mu,\ \pi(M,\cdot) = \nu \}$.

Optimal transportation seeks a solution $\pi \in \Pi(\mu,\nu)$ that produces the minimum total cost:
\begin{equation}
\label{eq:wd}
  W_p(\mu,\nu) \eqdef \bigg(\underset{\pi \in \Pi(\mu,\nu)}{\inf}  \int_{M\times M}^{} (c(x,y))^pd\pi(x,y)\bigg)^{\frac{1}{p}},
\end{equation}
\noindent
where $p$ indicates the finite moment of the cost function. The minimum cost is the \textit{p-Wasserstein distance}. In this paper, we only consider the \textit{2}-Wasserstein distance, $W_2$.

Monge's formulation restricts OT to preserve measures, that is, mass cannot be split during the mapping. Letting $T$ denote such a mapping, $T\colon x \rightarrow y$, we have $d\pi(x,y) \equiv d\mu(x) \delta(y-T(x))$. Therefore, we formally define $T$ as
\begin{equation}
  T_{opt} = \underset{T}{\arg\min}  \int_{M}^{} {c(x,T(x))}^{p}d\mu(x).
  \label{eq:ot_monge}
\end{equation}

\noindent 
In this paper, we follow (\ref{eq:ot_monge}). The details of the optimal transportation problem and the properties of the Wasserstein distance can be found in~\cite{villani2003topics,gibbs2002choosing}. With the abuse of notation, we use $\pi(\mu,\nu)$ to denote the Monge's OT map between $\mu$ and $\nu$ and since the map is applied to their supports $x$ and $y$ we also use $\pi\colon x\rightarrow y$ and $y=\pi(x)$ to denote the map.

\subsection{Variational Optimal Transportation}\label{sec:p_vot}
Suppose $\mu$ is continuous and $\nu$ is a set of Dirac measures in $M = \mathbb{R}^n$, supported on $\Omega_{\mu} = \{ x \in M\ |\ \mu(x)>0 \}$ and $\Omega_{\nu} = \{y_j \in M \ |\ \nu_j > 0\}, j = 1,...,k$, and their total measure equal: $vol(\Omega) = \int_{\Omega}d\mu(x) = \sum_{j=1}^{k}\nu_j$. \cite{gu2013variational} proposed a variational solution to this \textit{semi-discrete OT} on $\mathbb{R}^n$. It starts from a vector $\bm{h}=(h_1,...,h_k)^T$ and a piece-wise linear function: $\theta_{\bm{h}}(x) = \max\{\langle x,y_j\rangle+h_j\},\ j = 1,...,k$. Alexandrov proved in~\cite{alexandrov2005convex} that there exists a unique $\bm{h}$ that satisfies the following
\begin{equation}
\label{eq:alex}
    vol(x \in \Omega \mid \nabla \theta_{\bm{h}}(x)=y_j) = \nu_j.
\end{equation}
\noindent Furthermore, Brenier proved in~\cite{brenier1991polar} that $\nabla \theta_{\bm{h}}\colon x \rightarrow y$ \textit{is} the Monge's OT-Map if the transportation cost is the quadratic Euclidean distance $\| x - \nabla\theta_{\bm{h}}(x)\|_2^2$ .

Suppose $S_j(h) = \{x \in M\ |\ \nabla\theta_{\bm{h}}(x)=y_j\}$ is the projection of $\theta_{\bm{h}}$ on $\Omega$. Variational OT (VOT) solves
\begin{equation}
\label{eq:vot_energy}
\begin{split}
    E(\bm{h}) \eqdef & \int_{\Omega}^{}\nabla\theta_{\bm{h}}d\mu -\sum_{j=1}^{k}\nu_j h_j\\
    \equiv & \int_{\bm{0}}^{\bm{h}}\bigg(\sum_{j=1}^{k}\int_{\Omega\cap S_j(h)}d\mu\bigg) dh -\sum_{j=1}^{k}\nu_j h_j,
\end{split}
\end{equation} 
\noindent and thus converts the OT problem into searching in a vector space $\mathcal{H} = \{h \in \mathbb{R}^k\ |\ \int_{\Omega \cap S_j(h)}d\mu >0\ \text{for all}\ j \}$. Proved in~\cite{gu2013variational}, $E$ (\ref{eq:vot_energy}) is convex in $\mathcal{H}$ when $\sum_{j=1}^{k}h_j = 0$. The gradient of (\ref{eq:vot_energy}) is (\ref{eq:alex}). Thus, minimizing (\ref{eq:vot_energy}) when its gradient approaches $\bm{0}$ will give us the desired $\bm{h}$, and the map $\nabla\theta_{\bm{h}}$.

\subsection{Wasserstein Barycenters}\label{sec:p_wm}
Given a collection of measures and weights $\{\mu_i,\lambda_i\}_{i=1}^{N}$, there exists such a measure $\nu$ that the weighted average of the Wasserstein distances between $\nu$ and all $\mu_i$'s reaches the minimum. As exposed in~\cite{agueh2011barycenters}, Agueh and Carlier defined such a problem as finding a barycenter in the measure space with respect to the Wasserstein distance:
\begin{equation}\label{eq:wb}
    \nu = \underset{\nu \in \mathcal{P}_2(M)}{\arg\min} \sum_{i=1}^{N}\lambda_iW_2^2(\nu,\mu_i). \nonumber
\end{equation}

Wasserstein barycenters of discrete measures exist for mass splitting OT but may not for non-mass splitting or measure-preserving OT. Yet, proved in \cite{anderes2016discrete}, when the weights are uniform and all measures have finite number of supports, there still exists a barycenter $\nu$ that preserves the measure and whose number of supports $|\Omega_{\nu}|$ has a tight upper bound $|\Omega_{\nu}| \leq \sum_{i=1}^{N}|\Omega_{\mu_i}| - N + 1$, and the OT from every $\mu_i$ to $\nu$ preserves the measure.




\begin{algorithm}[!b]
\DontPrintSemicolon
\SetKwInOut{Input}{Input}\SetKwInOut{Output}{Output}
\Input{$\mu(x) \in \mathcal{P}(M)$ and Dirac measures  $\{\nu_j, y_j\}$}
    t = 0. \;
    \Repeat{{\normalfont convergence.}}{
        $\nu^{(t+1)} \leftarrow $ Update weight according to (\ref{eq:mw_nu}). \;
        $\pi^{(t+1)} \leftarrow $ Compute OT with fixed $y^{(t)},\nu^{(t)}$. \;
        $y^{(t+1)} \leftarrow $ Update support according to (\ref{eq:wm_y}). \;
        $t \leftarrow t + 1$.
    }
\KwRet{$\pi, y, \nu$}.\;
\caption{Wasserstein Means}
\label{alg:wm}
\end{algorithm}

\section{Wasserstein Means via Variational OT} \label{sec:mw}
A special case of the Wasserstein barycenters problem is when $N = 1$. In that case, we are computing a barycenter of a single probability measure. We call it the \textit{Wasserstein mean} (WM). Beyond a special case, the barycenters and the means have the following connection.
\begin{proposition}
Given a compact metric space $M$, a transportation cost $c(\cdot,\cdot)\colon M \times M \rightarrow \mathbb{R}^{+}$, and a collection of Borel probability measures $\mu_i \in \mathcal{P}(M)$, with weights $\lambda_i,\ i = 1,...,N$, the Wasserstein mean $\nu_m$ of their average measure induces a lower bound of the average Wasserstein distance from the barycenter $\nu_b$ to them, provided that $|\Omega_{\nu_b}| \leq |\Omega_{\nu_m}| \leq k$ for some finite $k$.
\end{proposition}
\begin{proof}
Since $W_2^2(\nu_b,\cdot)$ is convex for its metric property, according to Jensen's inequality, we have
\begin{equation}
    W_2^2(\nu_b,\sum_{i=1}^{N}\lambda_i \mu_i) \leq \sum_{i=1}^{N} \lambda_i  W_2^2(\nu_b,\mu_i). \nonumber
\end{equation}
\noindent Then, according to Wasserstein mean's definition, 
\begin{equation}
    W_2^2(\nu_m,\sum_{i=1}^{N}\lambda_i \mu_i) \leq  W_2^2(\nu_b,\sum_{i=1}^{N}\lambda_i \mu_i),\ \forall \nu_b. \nonumber 
\end{equation}
\noindent The result shows. The equal sign holds when $N = 1$. 
\end{proof}
\noindent We should point out that if $\{\mu_i\}$ are discrete measures, then for the barycenter to exist we need to add the condition from~\cite{anderes2016discrete} that $|\Omega_{\nu_b}| \leq \sum_{i=1}^{N}|\Omega_{\mu_i}| - N + 1$, which also bounds $|\Omega_{\nu_m}|$ through $|\Omega_{\nu_m}| \leq \sum_{i=1}^{N}|\Omega_{\mu_i}|$.

Now, approaching Wasserstein means is essentially through optimizing the following objective function:
\begin{equation} \label{eq:wm_2}
\begin{split}
   \min f(\pi,y,\nu) \eqdef & \underset{\pi, y_j, \nu_j}{\min} \sum_{j=1}^{k} \sum_{y_j = \pi(x)} \mu(x) \|y_j - x\|_2^2, \\
&    \text{s.t.}\ \  \nu_j = \sum_{y_j=\pi(x)}\mu(x).
\end{split}
\end{equation}
\noindent Compared to OT, solving WM w.r.t. (\ref{eq:wm_2}) introduces 2 additional parameters -- measure $\nu$ and its support $y$. When $y$ and $\nu$ are fixed, (\ref{eq:wm_2}) becomes a classic optimal transportation problem and we adopt variational optimal transportation (VOT)~\cite{mi2018variational} to solve it. Thus, (\ref{eq:wm_2}) is minimizing the lower bound of the OT cost.

Then, it boils down to solving for $y$ and $\nu$. Certainly (\ref{eq:wm_2}) is differentiable at all $y \in \mathbb{R}^{n \times k}$ and is convex. It's optimum w.r.t. $y$ can be achieved at 
\begin{equation} 
\label{eq:wm_y}
    \tilde{y_j} = \frac{\int_{\Omega_{\mu} \cap S_j} xd\mu(x)}{\int_{\Omega_{\mu} \cap S_j} d\mu(x)}.
\end{equation}
\noindent It is essentially to update the mean to the centroid of corresponding measures, adopted in for example~\cite{cuturi2014fast,ye2017fast,courty2017optimal}. The slight difference in our method is that VOT is non-mass splitting and thus the centroid in our case has a clear position without the need for weighting.
\ignore{\noindent For convenience, we use a function symbol $f$ to indicate the updating process. We write $f\colon y \rightarrow \hat{y}$, or using the iterator notation $t$, $f\colon y^{(t)} \rightarrow y^{(t+1)}$. Thus, the critical point of (\ref{eq:wm_2}) w.r.t. $y$ can be reached by $f\colon y \rightarrow \tilde{y}$.}

As discussed in \cite{cuturi2014fast}, (\ref{eq:wm_2}) is not differentiable w.r.t. $\nu$. However, we can still get its optimum through the following observation.

\begin{observation}
The critical point of the function $\nu \rightarrow f(\pi,\nu)$ is where $\nu$ induces $\pi$ being the gradient map of the unweighted Voronoi diagram formed by $\nu$'s support $y$. In that case, every empirical sample $\mu(x)$ at $x$ is mapped to its nearest $y_j$, which coincides with Lloyd's algorithm.
\end{observation}

\begin{proof}
Suppose $\nu$ induces the OT map $\pi$ from every $x$ to its nearest $y_j$. Then, the map $\pi'\colon x \rightarrow y_{j'}$ that satisfies any other $\nu'=\int_{\Omega \cap S_{j'}}d\mu(x)$ will yield an equal or larger cost $\int_{\Omega}\|y_j-x_i\|_2^2 d\mu(x_i) \leq \int_{\Omega}\|y_{j'}-x_i\|_2^2 d\mu(x_i)$.
\end{proof}
\noindent Thus, we can write the update rule for $\nu$ as 
\begin{equation}
\label{eq:mw_nu}
\begin{gathered}
    \tilde{\nu}(y_j) = \int_{\Omega \cap S_j}d\mu(x),\\
    \text{s.t.}\ S_j = \{x \in M\ |\ \|x-y_j\|_2 \leq \|x-y_i\|_2, i \neq j  \}.
\end{gathered}
\end{equation}

Updating the three parameters $\pi$, $y$, and $\nu$ can follow the \textit{block coordinate descent} method. Since at each iteration we have closed-form solutions in the $y$ and $\nu$ directions, there is no need to do a line search there. We wrap up our algorithm for computing the Wasserstein means in Alg.~\ref{alg:wm}

As discussed in~\cite{cuturi2014fast}, when $N=1$ and $p=2$, computing the Wasserstein barycenter (in this case the Wasserstein mean) is equivalent to Lloyd's k-means algorithm. The difference also occurs when we have a constraint on the weight $\nu_j(y)$. Ng~\cite{ng2000note} considered a uniform weight for all $S_j$. Our algorithm can adapt to any constraint on $\nu_j \geq 0$. In this case, our algorithm is equivalent to~\cite{cuturi2014fast} where the update of the support is equivalent to re-centering it by our (\ref{eq:wm_y}).

\begin{figure}[!t]    
  \centering
    \includegraphics[width=0.8\linewidth]{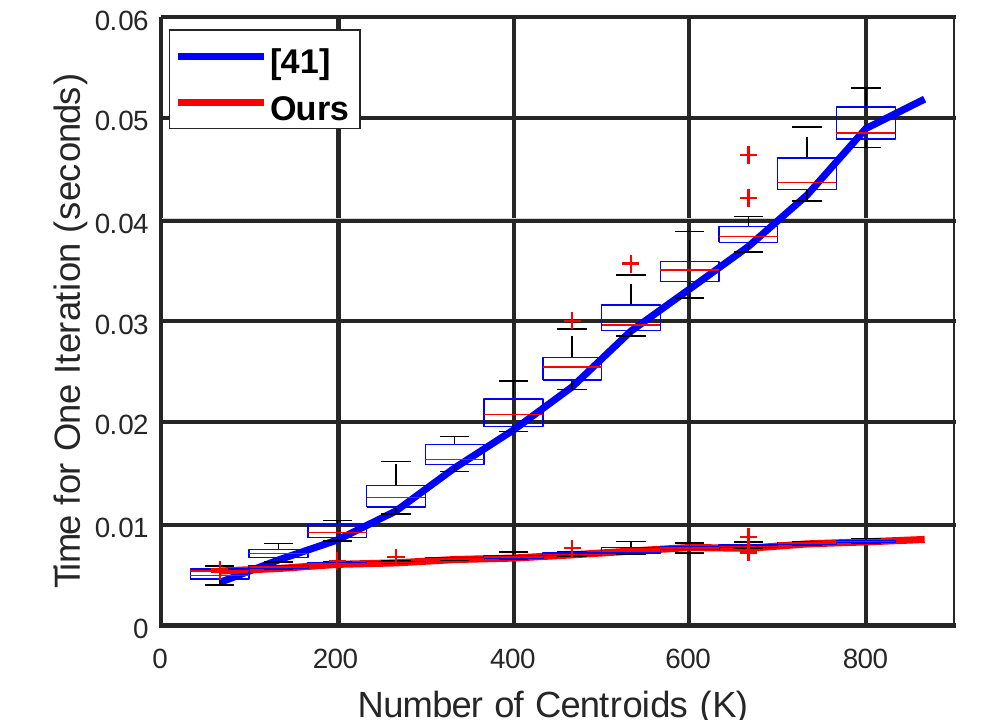}
    \caption{{\small Comparison of time over number of centroids.}}
    \label{fig:time}
\end{figure}

\textbf{Complexity} In practice, we use the total mass of the discrete measures inside each $S$. Then, we vectorize the computation with PyTorch because parameters in VOT, $\bm{h}$ (\ref{eq:vot_energy}), can be optimized individually and thus parallelly. Given $N$ empirical samples and $K$ centroids, our implementation of OT runs $\mathcal{O}(KN)$ on CPU and theoretically $\mathcal{O}(N)$ on GPU. Figure~\ref{fig:time} shows timing over $K \leftarrow [20:1000]$. $N=10,000$. The boxes along the plots come from 10 runs of 300 iterations for each $K$. The dimension of the data is 3. $y$ axis is in seconds per iteration. The plot shows the increased $K$ add few burden to RWM. The complexity added by regularization is as follows. The complexity in 5.1 is $\mathcal{O}(K)$; 5.3 is $\mathcal{O}(K^3)$ mainly from solving SVD, but in practice we choose a small or a constant number $K' << K$ for SVD; 5.4 is $\mathcal{O}(K)$ for computing curvature. Thus, the total computational complexity of RWM is $\mathcal{O}(N) + \mathcal{O}(K^3)$, depending on the regularization term. We also compute the pair-wise distances between empirical samples and centroids beforehand as in~\cite{cuturi2013sinkhorn}, making the memory consumption on the level of $\mathcal{O}(KN)$.

\begin{algorithm}[!b]
\DontPrintSemicolon
\SetKwInOut{Input}{Input}\SetKwInOut{Output}{Output}
\Input{$\mu(x) \in \mathcal{P}(M)$,  $\{\nu_j,y_j\}$}
    $t = 0$. \;
    \Repeat{$\pi$ and $y$ {\normalfont converge.}}{
        $\pi^{(t+1)} \leftarrow $ Compute OT $\pi(\mu,\nu)$ with fixed $y^{(t)}$. \;
        $\tilde{y} \leftarrow $ Compute new centroid according to (\ref{eq:wm_y}).\;
        \Repeat{ $y^{(t+1)}$ {\normalfont converges.}}{
            $y^{(t+1)} \leftarrow $ Update centroid by optimizing (\ref{eq:reg_wm_simple}).\;
        }
        $t \leftarrow t+1$.  \;
    }
\KwRet{$\pi,y$}.\;
\caption{Regularized Wasserstein Means}
\label{alg:rwm}
\end{algorithm}


\section{Regularized Wasserstein Means} \label{sec:rwm}
In many problems of machine learning, the solution that comes purely from the perspective of the mapping cost may not serve the best to represent the connection between origins and their images, let alone overfitting. Regularization is a common technique to introduce desired properties in the solution. In the previous section, we talked about the Wasserstein means problem and its optimizers: OT $\pi(\nu,\mu)$, support $y$, and the measure $\nu(y)$. In this section, we detail our strategies to regularize $y$ along with several regularization terms that we propose to penalize the Wasserstein means cost. For simplicity, we fix the given $\nu(y)$ in the following arguments and only consider $\pi$ and $y$ in the \textit{regularized Wasserstein means} (RWM) problem.

We start with a general loss function:
\begin{equation} 
\label{eq:reg_wm}
\begin{split}
    \mathcal{L}(\pi,y)  &= \mathcal{L}_{ot}(\pi,y) + \lambda\mathcal{L}_{\text{reg}}(y), \\
    \ignore{\text{where}\ \ } \mathcal{L}_{ot}(\pi,y)  &= \int_{\Omega}\|y-x\|_2^2 d\mu(x),\ \text{where}\ y=\pi(x).
\end{split}
\end{equation}
\noindent We call the first term the \textit{OT loss} or data loss. Our goal here is to explore $\mathcal{L}_{\text{reg}}(y)$ and the use of it. Optimizing (\ref{eq:reg_wm}) can also follow the block coordinate descent method. First, we fix the mean and compute the OT. Unlike in Alg.~\ref{alg:wm} where we directly update the mean to the average of their correspondences, next, we regularize the mean to satisfy certain properties through local minimization on (\ref{eq:reg_wm}).

Minimizing the OT loss $\mathcal{L}_{ot}(\pi,y)$ w.r.t. $y$ can be simplified to minimizing the quadratic loss for each support, i.e. $\mathcal{L}_{\tilde{y}} = \sum_j \|y_j-\tilde{y}_j\|_2^2 $, since they are equivalent:
\begin{equation}
\begin{split}
    \int_{S_j} \|y_j-x\|_2^2 d\mu(x) = (y_j^2 - 2y_j \int_{S_j}xd\mu(x) + C_1)  \\
     = \|y_j - \int_{S_j} x d\mu(x)\|_2^2 + C_2 = \|y_j - \tilde{y}_j\|_2^2 + C_2.
\end{split}
\end{equation}
\noindent $C_1, C_2$ are some constants. $\tilde{y}_j$ is from (\ref{eq:wm_y}) and $S_j$ is the set in which $x$ is mapped to $y_j$. It is defined by VOT as $S_j = \{x\ \in M |\langle y_j,x \rangle - h_j \geq \langle y_i,x \rangle - h_i\}, \forall i \neq j$, see Sec.~\ref{sec:p_vot}. Thus, we re-write (\ref{eq:reg_wm}) as 
\begin{equation}
\label{eq:reg_wm_simple}
    \mathcal{L}(\pi,y)  = \sum_j \|y_j-\tilde{y}_j\|_2^2 + \lambda\mathcal{L}_{\text{reg}}(y)
\end{equation}
Note, that $\mathcal{L}_{\text{reg}}$ undermines the metric properties of the Wasserstein distance and yet the distance is not our concern but the data term of the loss we designed for a broad range of applications. We provide the general algorithm to compute regularized Wasserstein means in Alg.~\ref{alg:rwm}. 


Citing the convergence proof from~\cite{grippo2000convergence}, as long as we add a convex regularization term, because $\pi\colon x\rightarrow y$ is compact and convex, our 2-block coordinate descent-based algorithm indeed will converge. In the rest of this section, we discuss in detail several regularization terms based on class labels, geometric transformation, and length and curvature, all of which are convex.

\subsection{Triplets Empowered by Class Labels}
\label{sec:rwm_label}
We begin with a fair assumption that samples of the same class reside closer to each other and samples that belong to different classes are relatively far away from each other. This behavior can be expressed by signed distances between samples. Given that, we propose to regularize the mean update process by adding a \textit{triplet} loss, promoting intra-class connection and discouraging inter-class connections.

The triplet loss was proposed in~\cite{schroff2015facenet}, inspired by~\cite{weinberger2009distance}. It targets the metric learning problem which is finding an embedding space where samples of the same desired property reside close to each other and vise versa. In triplets, samples are characterized into three types -- \textit{anchor}, \textit{positive}, and \textit{negative}, denoted as $y^a$, $y^p$, and $y^n$. The motivation is that the anchor is closer by a margin of $\alpha$ to a positive than it is to a negative:
\begin{equation}\nonumber
    \mathcal{L}_{\text{reg}}(y) = \sum^{K}_{i}[\|y_j^a - y_j^p\|^2_2 - \|y_j^a - y_j^n\|^2_2 + \alpha]_{+}.
\end{equation}
\noindent The overall RWM loss w.r.t. $y$ (\ref{eq:reg_wm_simple}) becomes 
\begin{equation}
\label{eq:triplet_wm}
\begin{split}
        \mathcal{L}(y) & =  \sum_j \|y_j-\tilde{y}_j\|_2^2 + \lambda \mathcal{L}_{\text{triplet}}(y).
\end{split}
\end{equation}

\label{sec:exp_gm}
\begin{figure}[!t]    
  \centering
    \includegraphics[width=\linewidth]{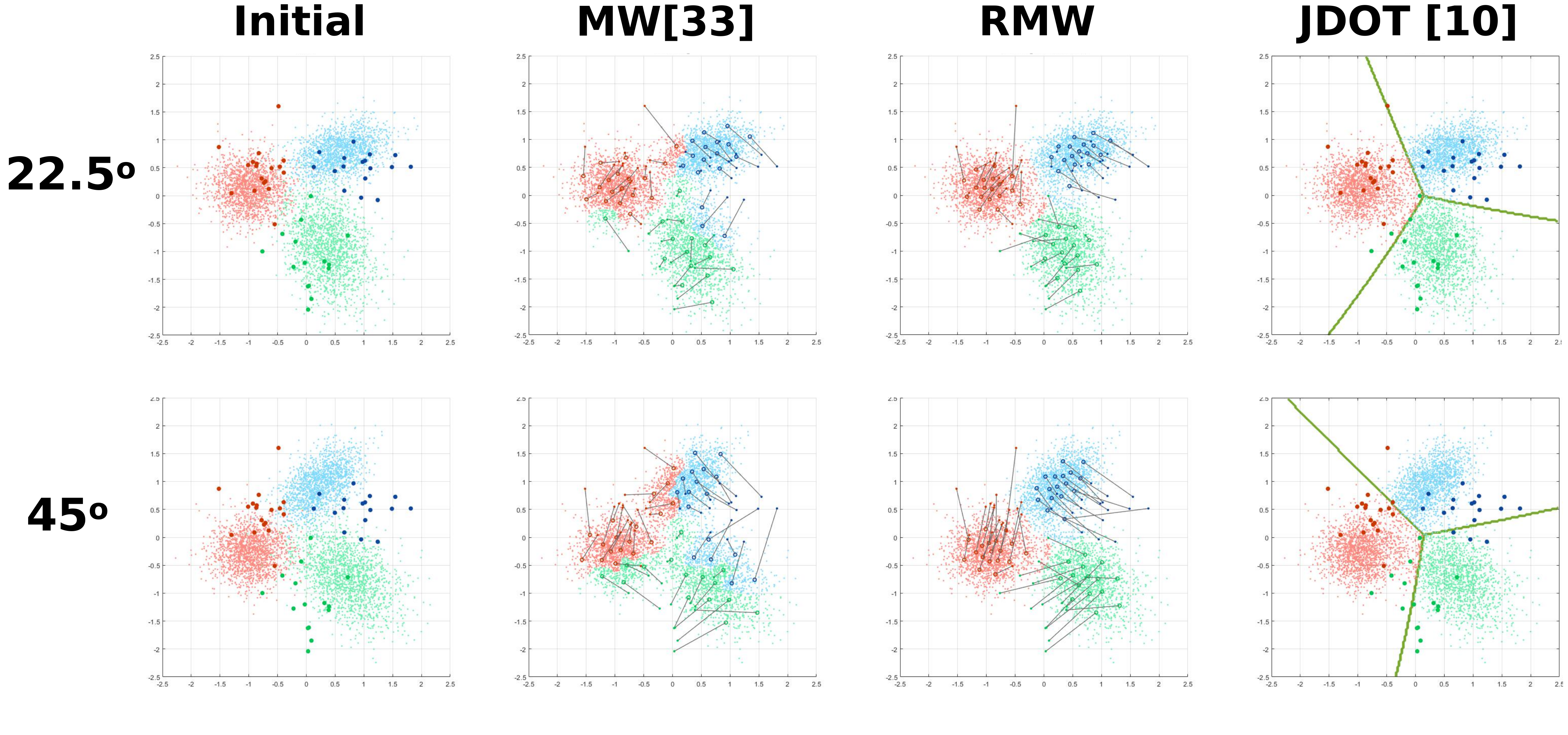}
    \caption{{\small Regularizing the WM by the intra-class triplets can adapt it to domains that suffer unknown rotations.}}
    \label{fig:gm}
\end{figure}

Fig.~\ref{fig:gm} shows an example of aligning Gaussian mixtures by (\ref{eq:triplet_wm}). Suppose a mixture has three components with different parameters, each belonging to a different class shown in three colors. We rotate the mixture by a certain degree to emulate an unknown shift and apply our method to recover the shift.

We sample the source domain 50 times and the target domain 5,000 times at $22.5^o$ and $45^o$. Fig.~\ref{fig:gm} 1$^{\text{st}}$ column shows the setups. The 2$^{\text{nd}}$ column shows the result from computing the WM without regularization as in~\cite{mi2018variational}. The 3$^{\text{rd}}$ column shows our result. Our method can well drive source samples into the correct target domain. The lighter colors on the target samples in the 2$^{\text{nd}}$ column indicate the predicted class by using the OT correspondence. Since our OT preserves the measure during the mapping, we can deterministically label each unknown sample by querying its own centroid's class. Note, that this is equivalent to the 1NN classification algorithm based on the \textit{power Euclidean distance}~\cite{mi2018variational}. Only when the weight of every centroid equals each other will the power distance coincide with the Euclidean distance. In the last column, we show the result from~\cite{courty2017joint}. It learns an RBF SVM classifier on the target samples.

\begin{figure}[!t]
  \centering
    \includegraphics[width=\linewidth]{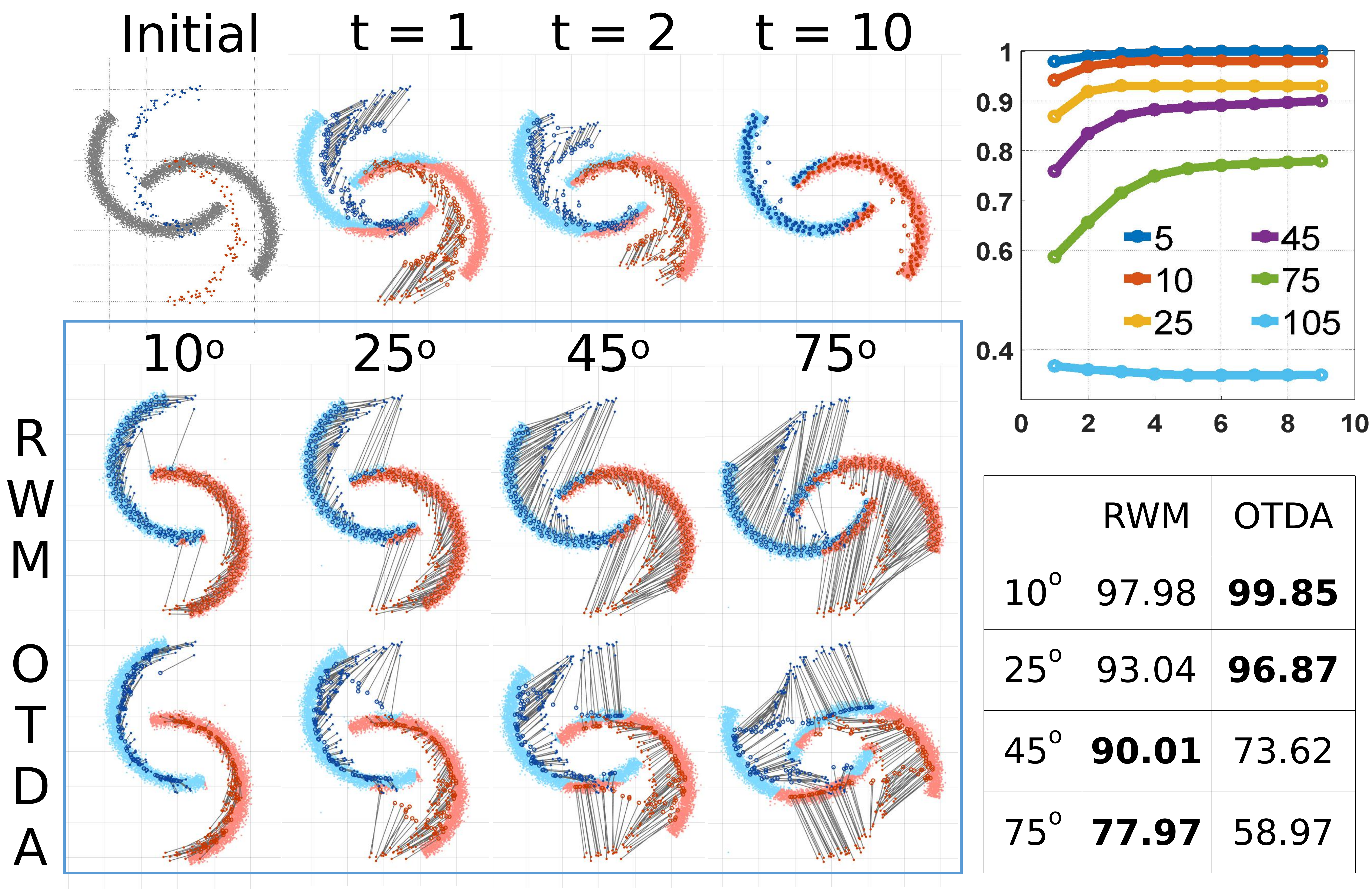}
    \caption{{\small RWM adapting shifted two moons: 1$^{\text{st}}$ row performance over iteration under 45$^o$; 2$^{\text{nd}}$ and 3$^{\text{rd}}$ rows performances of RWM and OTDA under different degrees.}}
     \label{fig:twomoon}
\end{figure}

\subsection{Geometric Transformations}
\label{sec:rwm_affine}
While OT recovers a transformation between two domains that induces the lowest cost, it does not consider the structure within the domains. Pre-assuming a type of the transformation and then estimating its parameters is one of the popular approaches to solving domain alignment-related problems, for example in~\cite{gopalan2011domain,courty2017optimal}. In this way, the structure of the domain can be preserved to some extent. Let us follow this trend and assume that two domains can be matched by a geometric transformation with modifications, that is, any transformation between domains is a combination of a parametric geometric transformation and an arbitrary transformation. This leads to our following strategy that we, on the one hand, regularize the mean to be roughly a geometric transformation in order to preserve the structure of the source domain during the mapping but on the other hand also allow OT to adjust the mapping so that it can recover irregular transformations.

We follow Alg.~\ref{alg:rwm}. First, compute OT to obtain the target mean positions $\tilde{y} = \pi(x)$ and then use the paired means $\{y,\tilde{y}\}$ to determine the parameters of a geometric transformation $\mathcal{T}$ subject to $\tilde{y} = \mathcal{T}y$ through a least squares estimate. Suppose $y_j^{\mathcal{T}} = \mathcal{T}y$ is the estimate purely based on the affine transformation, then, we have the RWM loss
\begin{equation}
\label{eq:reg_affine}
    \mathcal{L}(\pi,y)  = \sum_j \|y_j-\tilde{y}_j\|_2^2 + \lambda \sum_j \|y_j - y_j^{\mathcal{T}}\|_2^2.
\end{equation}
\noindent Candidates of the geometric transformations include but not limited to perspective, affine, and rigid transformations.

We demonstrate (\ref{eq:reg_affine}) with \textit{two moons} in Fig.~\ref{fig:twomoon}. The known domain contains $200$ samples in blue and red. The unknown domain is the known domain after a rotation, sampled $10,000$ times in grey. We assume the prior is a rigid transformation. The top row shows the result on the $45^o$ case after several iterations. In the end, RWM almost recovers the transformation with a small error. Top right shows accuracy over iterations under different degrees. The 2$^{\text{nd}}$ row shows the result under different degrees of rotation. We weight in OTDA-GL's result~\cite{courty2017optimal} in the 3$^{\text{rd}}$ row showing RWM's superiority over OTDA under large transformations and its inferiority under small transformations. We also notice that RWM maps the samples \textit{into} the domain which OTDA fails to.


\subsection{Topology Represented by Length and Curvature}
\label{sec:rwm_curve}
The nature of many-to-one mapping in the WM problem enables itself to be suitable for skeleton layout. Consider a 3D thin, elongated point cloud. Our goal is to find a 3D curve consisting of sparse points to represent the shape of the cloud. The problem with directly using WM for skeleton layout is that the support is unstructured. Therefore, we propose to pre-define the topology of the curve and add the length and curvature to regularize its geometry, both intrinsically (length) and extrinsically (curvature).


We give an order of the support so that they can form a piece-wise linear curve. For each three adjacent supports, $y_{j-1},y_{j},y_{j+1}$, we fit a quadratic spline curve $\gamma(t)$ of 100 points. Its length is approximated by summarizing the length segment $\int_0^{length} ds=\int_0^1\|\gamma'(t)\|dt$, and its curvature at the middle point $y_i$ can be approximated by the total curvature $\int_0^{length} \mathcal{K}^2(t) ds$, $\mathcal{K}(t)=\frac{\|\gamma'(t)\times\gamma''(t)\|}{\|\gamma'(t)\|^3}$ as in~\cite{ulen2015shortest}. Thus, the regularization on the length and curvature can express itself as follows:
\begin{equation}\label{eq:reg_length_curvature}
\begin{split}
   \lambda\mathcal{L}_{\text{reg}} =  \lambda_1\sum_{1\leq i < k}g(\gamma'(y_i)) + \lambda_2\sum_{1 < i < k}l(\gamma''(y_i)). 
\end{split}
\end{equation}
\noindent where $g(\cdot)$ and $l(\cdot)$ are some functions computed out of the length and curvature based on $y$, which are both convex making (\ref{eq:reg_length_curvature}) convex. We could go further and include torsion into the term but since we do not pursue a perfectly smooth curve but rather the reasonable embedding of the supports in the interior of the point cloud, we have passed torsion. 

In case the shape have branches, we can easily extend (\ref{eq:reg_length_curvature}) considering the skeleton as a whole when computing the OT and regularizing each branch separately. Suppose, now, the skeleton $\Gamma=\{\gamma_j\}$ is a set of 1-D curves. Finally, we propose the following loss for skeleton layout:
\begin{equation} \label{eq:reg_curve}
\begin{split}
   &\mathcal{L}(\pi,y) = \sum_j \|y_j-\tilde{y}_j\|_2^2 \\
     &+ \sum_{\gamma \in \Gamma}\big( \lambda_1\sum_{1\leq i < k}g(\gamma'(y_i)) + \lambda_2\sum_{1 < i < k}l(\gamma''(y_i)) \big).
\end{split}
\end{equation}

\bgroup
\def\arraystretch{1}
\begin{table}[b]
\centering
\caption{Classification Accuracy (\%) on Office-31 W $\rightarrow$ A } \label{tab:office}
\begin{tabular}{c|@{\ \ }c@{\ \ }c@{\ \ }c@{\ \ }c}
\hline
Feature   & 1NN          & WM           & OTDA         & RWM   \\
\hline
Decaf-fc6 & 30.2$\pm$1.3 & 32.7$\pm$2.3 & 33.9$\pm$2.1 & \textbf{36.4}$\pm$2.7\\
Decaf-fc7 & 31.3$\pm$1.9 & 34.6$\pm$2.2 & 35.8$\pm$1.5 & \textbf{43.2}$\pm$2.6\\
\end{tabular}
\end{table}
\egroup

\section{Applications}
\label{sec:exp}
We demonstrate the use of RWM in domain adaptation (class label), point set registration (geometric transformation), and skeleton layout (topology).

\subsection{Domain Adaptation}\label{sec:exp_da}

We evaluate our method on the office-31 dataset~\cite{saenko2010adapting}. Office-31 includes two subsets -- Amazon and Webcam. We adapt from Webcam to Amazon (W $\rightarrow$ A). The Amazon set contains 2,848 images from 31 categories. Each category has a different number of samples from 36 to 100. The Webcam set archives 826 images from the same 31 categories, each having between 11 to 43 samples. 

\begin{figure}[t]    
  \centering
    \includegraphics[width=\linewidth]{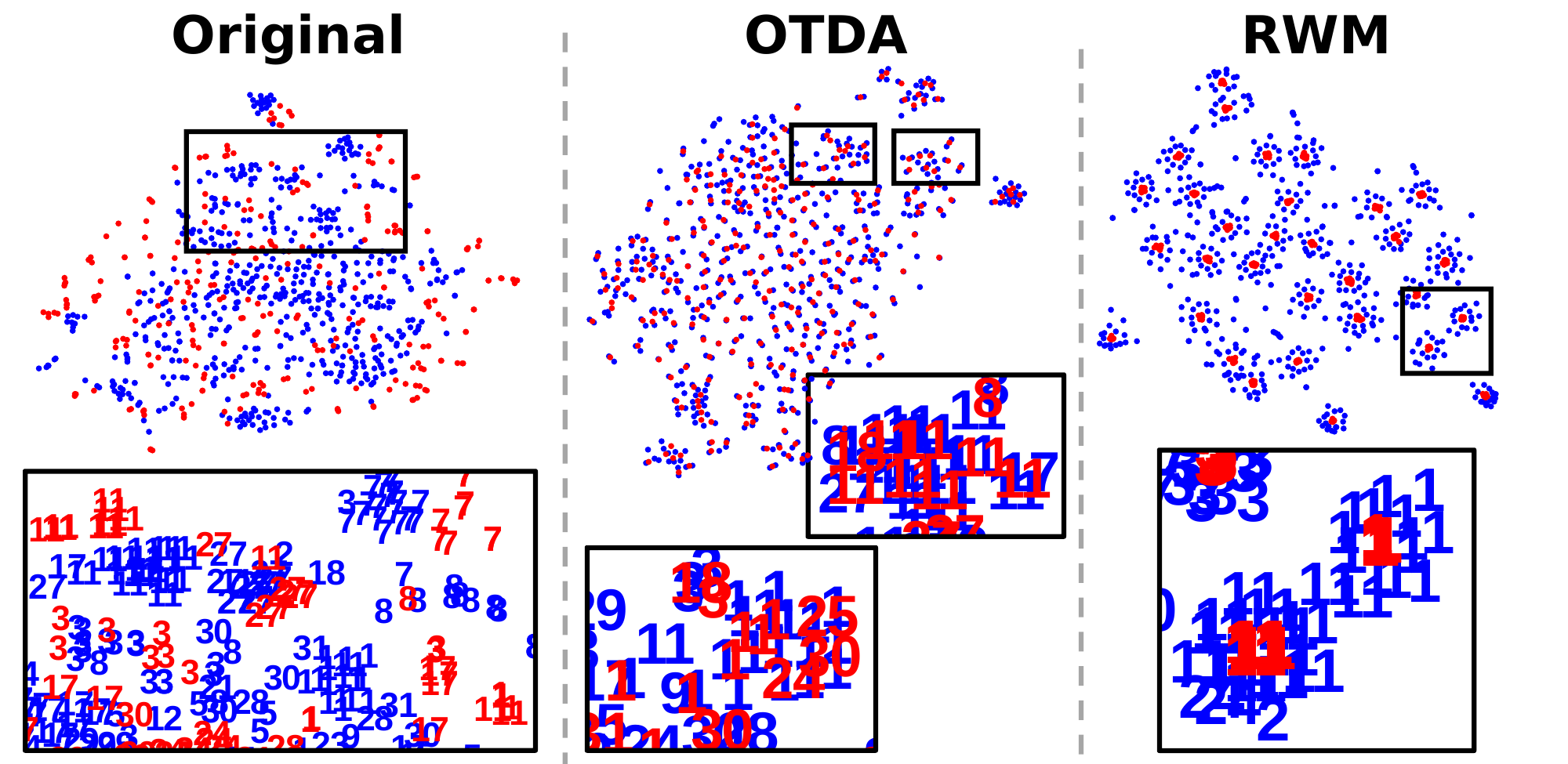}
    \caption{{\small t-SNE plots of Office samples after OTDA and RWM.}}
    \label{fig:office_tsne}
\end{figure}

We use the Decaf-fc6 and Decaf-fc7 features provided along with the dataset. Each sample now is encoded into a vector of 4,096 dimensions. The setup is similar to OTDA~\cite{courty2017optimal}. We randomly select 20 samples per class from Amazon and 10 samples per class from Webcam, because the `ruler' category of Webcam only has 11 samples and we want each class to have an equal number of samples. Then, we normalize the weight of the sample so that the total weight from Amazon and from Webcam are both one. Each sample is assumed to have an equal weight: Amazon sample $1/620$ and Webcam sample $1/310$.


We compare RWM with OTDA and also include 1NN and the original WM as baselines. The experiments are repeated 10 times and Tab.~\ref{tab:office} summarizes the averaged results. RWM outperforms other methods by a large margin. We also show the resulting t-SNE embeddings in Fig.~\ref{fig:office_tsne}. From left to right are the original embeddings, embeddings after OTDA, and embeddings after RWM. Blue dots represent Amazon samples and red dots Webcam samples. Numbers indicate classes. RWM successfully cluster samples from the same class into distinguishable clusters while OTDA on the other hand very well integrates the source domain into the target domain (but with larger errors). Zoom in the pictures to see the samples of 1, `bike', and 11, `keyboard'. The regularization weight of OTDA Laplacian is $0.3$. It is from a search in $\{1, 0.3, 0.1, 0.03, 0.01\}$. The weight of RWM is $1$ from a search in $\{3, 1, 0.3, 0.1, 0.03, 0.01\}$



\subsection{Point Set Registration}\label{sec:exp_pcr}
Registering point sets is key to many downstream applications such as surface reconstruction and stereo matching. Point set registration algorithms aim to assign correspondences between two sets of points and to recover the transformation between them~\cite{myronenko2010point}. Figure~\ref{fig:icp} left shows a Stanford Bunny in a grey point set and its shifted version in a colored point set after a random noisy translation and a rotation. We apply (\ref{eq:reg_affine}) to recovering the transformation. With this example, we also test our algorithm under the extreme condition when we have the same number of empirical samples and centroids. Our algorithm RWM still produces a one-to-one map between the two point sets. The transformation then perfectly aligns them while the traditional iterative closest point (ICP) algorithm fails to recover the transformation. The reason is that ICP assigns the correspondence based on nearest neighbors while RWM uses OT which considers the point set as a whole when computing the correspondence. Note, that by pre-defining the regularization as a rigid transformation and adjusting its weight, we can perform both rigid and non-rigid registration. In the above example, the regularization weight is $\lambda = 10$. Our alignment technique might be further incorporated into e.g. \cite{yang2016go} for globally optimal alignment.

\begin{figure}[t]    
  \centering
    \includegraphics[width=0.95\linewidth]{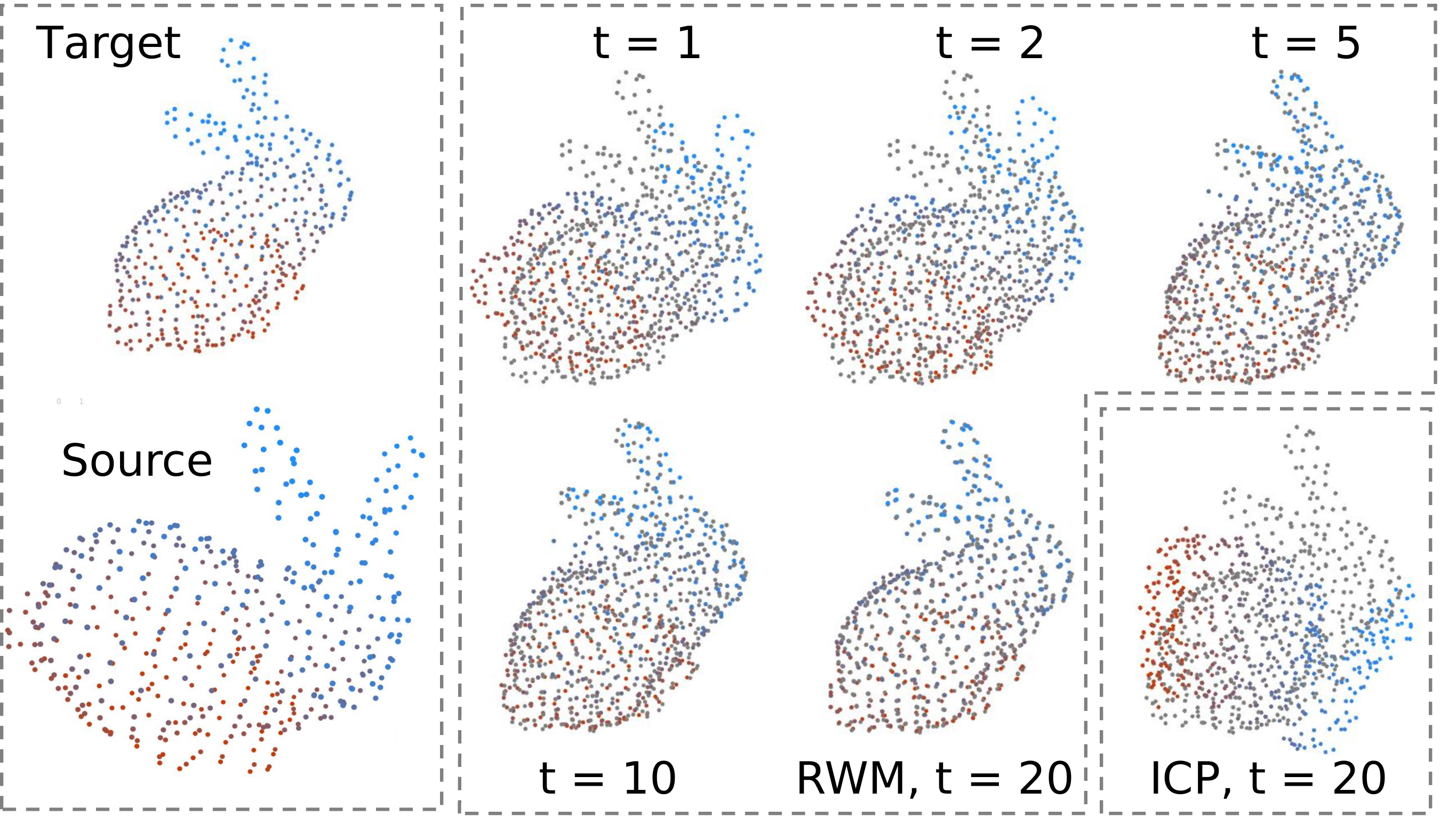}
    \caption{{\small Alignment of translationally and rotationally shifted bunnies after RWM and ICP. t indicates the number of iterations.}}
    \label{fig:icp}
\end{figure}

\subsection{Skeleton Layout}\label{sec:e_sl}
\ignore{We follow a similar setup as in \cite{solomon2015convolutional}}. Suppose we have a point cloud $\mu \in \mathcal{P}(\mathbb{R}^3)$ and a graph $G=(V,E)$ representing the topology of the shape. Then, the problem is finding particular embeddings of the nodes $y(\nu): \nu \rightarrow \mathbb{R}^{3}$ that can relate the graph to the geometry of the point cloud. \ignore{On the one hand, we expect the skeleton to have (i) a small reconstruction error with the original cloud; on the other hand, the skeleton itself has to have (ii) a clear shape and (iii) certain smoothness. The three terms of (\ref{eq:reg_curve}) respectively satisfy these requirements.}

Now, consider the human shape point cloud in Fig.~\ref{fig:skel} top left. We initial a rough embedding of a graph by fixing its ends $V_0 \subset V$ to certain known positions $y_{\nu \in V_0}$ which are head, hands, and feet in this example, and set the rest of nodes evenly distribute along their branches. Our goal is to embed the nodes $\nu \in V \backslash V_0$ in this $\mathbb{R}^3$ space by applying (\ref{eq:reg_curve}). Because the weight of each centroid determines its boundaries with other centroids, it has to be adjusted to the local density of the cloud so that all the centroids could roughly evenly lay on the skeleton. Thus, we relax the restriction on weight and reinstate (\ref{eq:mw_nu}). We update the weight by momentum gradient descent, $\nu(y_j) ^{(t+1)} \leftarrow \lambda \nu(y_j) ^{(t)} + (1-\lambda) \int_{\Omega \cap S_j}d\mu(x)$ to prevent it from quickly trapped into a local minimum like k-means.

Top right of Fig.~\ref{fig:skel} shows our result. The skeleton successfully captures the shape of the point cloud. Colors of the skeleton nodes based on their position in the graph are transferred to the surface according to their OT correspondences. We compare the result from Lloyd's k-means algorithm and with RW in the 2$^{nd}$ and 3$^{rd}$ columns. Equal weight of regularization is added to Lloyd's algorithm to make it a fair comparison. We also test our method in an extreme initial condition. As shown in (b), our algorithm eventually recovers a coherent, correct shape, but without the regularization we could end up with ``ill-posed'' embeddings. The figure also writes the mean square errors (MSE). Our method achieves small MSEs while maintaining the topology. In the bottom left, we show result from Stanford Armadillo. In the bottom right, we show the result from~\cite{solomon2015convolutional} as the ground truth. It regards the problem as a \textit{Wasserstein propagation} problem and adopted Wasserstein barycenter techniques to relate the samples of the cloud to the graph, which is much heavier. The average time of 5 trials by~\cite{solomon2015convolutional} was 1,200 seconds while ours took 15 seconds. CPU: Intel i5-7640x 4.0 GHz.


\begin{figure}[!t]
  \centering
    \includegraphics[width=0.95\linewidth]{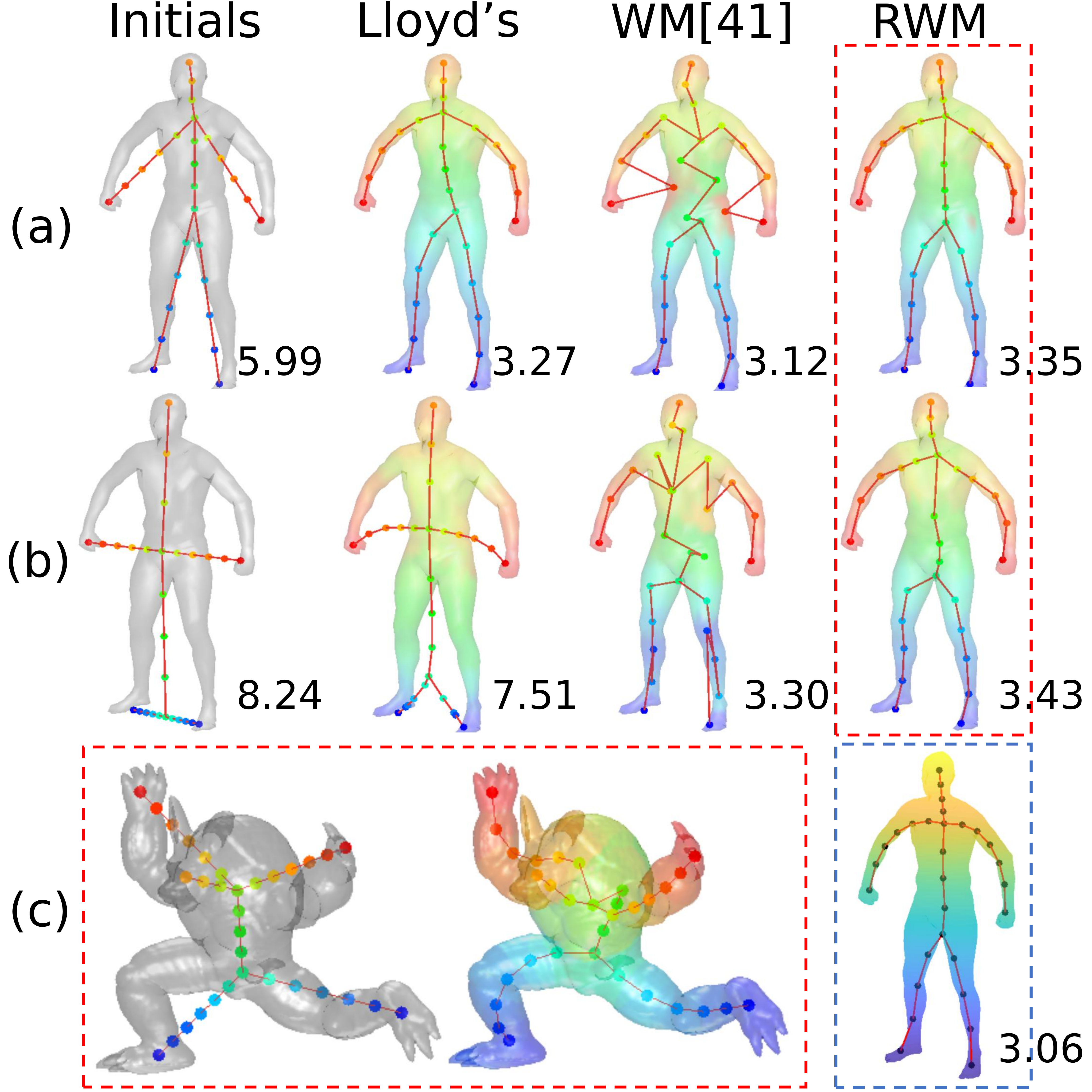}
    \caption{{\small Skeleton layout. RWM embeds a pre-defined graph which relates to the shape of the cloud. Numbers indicating MSE showing RWM balances between MSE and topology.}}
     \label{fig:skel}
\end{figure}

\section{Conclusion}
\label{sec:discuss}
We have talked about the Wasserstein means problem and our method to regularize it. The results have shown that our method can well adapt to different problems by adopting different regularization terms. This work opens up a new perspective to look at the Wasserstein means problem, or the k-means problem, as well as regularizing them.

We expect further use of regularized optimal transportation techniques on aligning distributions in high-dimensional spaces. Future work in our line of research could also include regularizing the barycenters.

\subsubsection{Acknowledgements} 
This work was supported in part by NIH (RF1AG051710 and R01EB025032). Liang Mi is supported in part by ASU Completion Fellowship.

\bibliographystyle{aaai}
\bibliography{main}
\end{document}